\documentclass[twocolumn]{svjour3}          % twocolumn
\smartqed  % flush right qed marks, e.g. at end of proof

\usepackage{graphicx}

\usepackage{geometry}                
\usepackage{mathptmx}
\usepackage{graphicx}
\usepackage{amssymb}
\usepackage{epstopdf}
\usepackage[applemac]{inputenc} 
\usepackage{tabularx}
\usepackage{psfrag} 
\usepackage{float}
\usepackage{algorithmicx}
\usepackage{algorithm}
\usepackage[noend]{algpseudocode}

\usepackage[authoryear]{natbib}

\usepackage{amsmath}
\usepackage{stmaryrd}

\usepackage{hyperref}
\usepackage{breakurl}

\usepackage{graphicx}
\usepackage{caption}
\usepackage[toc,page]{appendix}

\newcommand{\cconv}{{\,\tilde \star\,}}
\newcommand{\oconv}{{\,\bar \star\,}}
\newcommand{\tconv}{{\,\tilde \star\,}}
\newcommand{\SE}{{SE(2)}}
\newcommand{\SO}{{SO(2)}}
\newcommand{\R}{\mathbb{R}}
\newcommand{\ox}{\bar x} 
\newcommand{\V}{{\bf V}} 
\newcommand{\oy}{\bar y} 
\newcommand{\oh}{\bar h} 
\newcommand{\og}{\bar g} 
\newcommand{\ophi}{\overline \phi} 
\newcommand{\opsi}{\overline \psi} 
\newcommand{\tx}{\tilde x} 
\newcommand{\tiy }{\tilde y} 
\newcommand{\tphi}{\widetilde \phi} 
\newcommand{\tpsi}{\widetilde \psi} 

\newcommand{\tW}{{\widetilde W}}
\newcommand{\tU}{{\widetilde U}}
\newcommand{\tS}{{\widetilde S}}
\newcommand{\oW}{{\overline W}}
\newcommand{\oDownarrow}{{\overline \downarrow}}

\newcommand{\Z}{\mathbb{Z}}
\newcommand{\C}{\mathbb{C}}
\newcommand{\N}{\mathbb{N}}
\newcommand{\E}{\mathbb{E}}
\newcommand{\cPP}{\Lambda}
\newcommand{\la}{{\lambda}}

\newcommand{\tLa}{{\tilde \Lambda}}

\newcommand{\lau}{{\lambda_1}}

\newcommand{\mm}{{M}}
\newcommand{\om}{{\omega}}
\newcommand{\equaldef}{\triangleq}

\newcommand{\LD}{{{\bf L^2}(\mathbb R^2)}}

\newcommand{\Ld}{{{\bf L^2}}}

\DeclareMathOperator{\spn}{span}

\setcitestyle{authoryear}
\journalname{Submitted to International Journal of Computer Vision}
\begin{document}\sloppy

\title{Rigid-Motion Scattering for Texture Classification%
\thanks{Work supported by ANR 10-BLAN-0126 and Advanced ERC InvariantClass 320959}
%about the article that should go on the front page should be
%placed here. General acknowledgments should be placed at the end of the article.}
}
%\subtitle{Do you have a subtitle?\\ If so, write it here}

%\titlerunning{Short form of title}        % if too long for running head

\author{Laurent Sifre         \and
        Stéphane Mallat
}

%\authorrunning{Short form of author list} % if too long for running head

\institute{L. Sifre \at
              CMAP Ecole Polytechnique \\
              Route de Saclay, 91128 Palaiseau France \\
              \email{laurent.sifre@gmail.com}           %  \\
%             \emph{Present address:} of F. Author  %  if needed
           \and
           S. Mallat \at
           Département d'informatique\\ 
           École normale supérieure \\
			45 rue d'Ulm F-75230 Paris Cedex 05 France
}

\date{Received: date / Accepted: date}
% The correct dates will be entered by the editor

\maketitle

\begin{abstract}
A rigid-motion scattering computes adaptive invariants along
translations and rotations, with a deep convolutional network.
Convolutions are calculated
on the rigid-motion group, with wavelets defined on the
translation and rotation variables. It preserves joint rotation and
translation information,
while providing global invariants at any desired scale.
Texture classification is studied, through the characterization of 
stationary processes from a single realization.  
State-of-the-art results are obtained on multiple texture
data bases, with important rotation and scaling variabilities.
\keywords{Deep network \and scattering \and wavelet \and rigid-motion \and texture \and classification}
% \PACS{PACS code1 \and PACS code2 \and more}
% \subclass{MSC code1 \and MSC code2 \and more}
\end{abstract}

\section{Introduction}

Image classification requires to find representations which reduce
non-informative intra-class variability, and hence which are partly 
invariant, while preserving 
discriminative information across classes. Deep neural networks build
hierarchical invariant representations by applying
a succession of linear and non-linear operators which are learned from
training data. They provide state of the art results
for complex image classifications tasks
\citep{hinton, lecun, sermanet, alex, google}. A major issue is to 
understand the properties of these networks, what 
needs to be learned and what is generic and
common to most image classification problems. 
Translations, rotations and scaling are common
sources of variability for most images, because of changes of view points and
perspective projections of three dimensional surfaces. Building 
adaptive invariants
to such transformation is usually considered 
as a first necessary steps for classification \citep{poggio}. 
We concentrate on this generic part, which is
adapted to the physical properties of the imaging environment, as opposed to
the specific content of images which needs to be learned.

This paper defines deep convolution scattering networks which can provide
invariant to translations and rotations, and hence
to rigid motions in $\R^2$. The level of invariance is adapted to the
classification task. Scattering transforms have been introduced to
build translation invariant representations, which are stable to deformations
\citep{mallat}, with applications to image classification \citep{joan}. 
They are implemented as a convolutional network, with successive
spatial wavelet convolutions at each layer.
Translations is a simple commutative
group, parameterized by the location of the input pixels. Rigid-motions 
is a non-commutative group whose parameters are not explicitly given by the
input image, which raises new issues. The first one is to understand
how to represent the joint information between translations and rotations.
We shall explain why separating both 
variables leads to important loss of information and yields representations
which are not sufficiently discriminative. This leads to the
construction of a scattering transform on the full rigid-motion group,
with rigid-motion convolutions on the joint rotation and translation
variables. 
Rotations variables 
are explicitly introduced in the second network layer,
where convolutions are performed on the rigid-motion group along the
joint translation and rotation variables. As opposed to translation
scattering where linear transforms are performed along spatial variables
only,
rigid-motion scattering recombines the new variables
created at the second network layer, which is usually done in deep
neural networks. However, a rigid-motion scattering involves
no learning since 
convolutions are computed with predefined wavelets 
along spatial and rotation variables. The stability is guaranteed by 
its contraction properties, which are explained.

We study applications of rigid-motion scattering
to texture classification, where translations, rotations and scaling 
are major sources of variability.
Image textures can be modeled
as stationary processes, which are typically
non Gaussian and non Markovian, with long range dependencies.
Texture recognition is a fundamental problem of 
visual perception, 
with applications to medical, satellite imaging, material recognition \citep{uiuc,wmfs,srp}, object or scene recognition \citep{malikscene}. 
Recognition is performed from a single image, and hence can not involve
high order moments, because their 
estimators have a variance which is too large. 
Finding a low-variance ergodic representation,
which can discriminate
these non-Gaussian stationary processes, is a fundamental probability
and statistical issue.

Translation invariant
scattering representation of stationary processes have been studied 
to discriminate texture which do not involve
important rotation or scaling variability \citep{joan,joan2}. 
These results are extended to joint translation and rotation invariance.
Invariance to scaling variability is incorporated through linear
projectors. It provides effective invariants, which yield
state of the art classification
results on a large range of texture data bases.

Section \ref{Transsec} reviews the construction of translation invariant
scattering transforms. Section \ref{subsec:separable_vs_joint} explains 
why invariants to rigid motion can not be computed by separating the
translation and rotation variables, without loosing important information.
Joint translation and rotation operators defines
a rigid motion group, also called
special Euclidean group.
Rigid-motion scattering transforms are studied in Section 
\ref{sec:roto-trans-scat}. Convolutions on the rigid-motion
group are introduced in Section \ref{subsec:affg} in order to
define wavelet tranforms over this group. Their properties are
described in Section \ref{subsec:wavelet_transform_roto_trans}.
A rigid-motion scattering 
iteratively computes the modulus of such wavelet transforms.
The wavelet transforms jointly process translations and
rotations, but can be computed with separable convolutions
along spatial and rotation variables. 
A fast filter bank implementation is described in Section \ref{sec:fawt},
with a cascade of spatial convolutions and downsampling.
Invariant scattering representations are applied to image
texture classification in Section \ref{sec:classification}. 
State of the art results 
on four texture datasets containing different types and ranges of variability \citep{KTH_TIPS_DL, UIUC_DL, UMD_DL, FMD_DL}. All numerical experiments are reproducible with the ScatNet \citep{scatnet} MATLAB toolbox.

\section{Invariance to Translations, Rotations and Deformations}
\label{Transsec}

Section \ref{subsec:translat} 
reviews the property of translation invariant representations and
their stability relatively to deformations. The use of wavelet
transform is justified because of their stability to deformations.
Their properties are summarized in Section \ref{subsec:wavelet2d}.
Section \ref{subsec:scat2d} describes translation scattering transforms,
implemented with a deep convolutional network. Separable extensions
to translation and rotation invariance is discussed in 
Section \ref{subsec:separable_vs_joint}. It is shown that this
simple strategy leads to an important loss of information.

\subsection{Translation Invariance and Deformation Stability}
\label{subsec:translat}

Building invariants to translations and small deformations is a prototypical
representation issue for classification, which carries major ingredients
that makes this problem difficult.
Translation invariance is simple to compute. There are many possible strategies
that we briefly review. The main difficulty is to build a representation
$\Phi(x)$ which is also stable to deformations. 

A representation $\Phi (x)$ is said to be translation invariant
if $x_v (u) = x(u-v)$ has the same representation
\[
\forall v \in \R^2~~,~~\Phi(x) = \Phi(x_v)~.
\]
Besides translation invariance, it is often necessary to build
invariants to any specific class of deformations through linear
projectors. Invariant to translation can be computed with a registration
$\Phi x(u) = x(u - a(x))$ where $a(x)$ is
an anchor point which is translated when $x$ is translated.
It means that if $x_v (u) = x(u-v)$ then
$a(x_v) = a(x) + v$. For example,   
$a(x) = \arg \max_{u} |x \star h(u)|$,
for some filter $h(u)$. These invariants are simple and preserve
as much information as possible. The Fourier transform
modulus $|\hat x(\omega)|$ is also invariant to translation.

Invariance to translations is often not enough. 
Suppose that $x$ is not just translated but also 
deformed to give $x_\tau(u) =
x(u-\tau(u))$ with $|\nabla \tau(u)| < 1$. 
Deformations belong to the infinite dimensional group of diffeomorphisms. 
Computing invariants to deformations would mean losing too much information. In a digit classification problem, a deformation invariant representation would
confuse a 1 with a 7. We then do not want to be invariant to any deformations,
but only to the specific deformations within the digit class, while preserving information to discriminate different classes. Such deformation invariants need to be learned as
an optimized linear combinations.

Constructing such linear invariants requires the
representation to be stable to deformations. 
A representation $\Phi(x)$ is stable to deformations 
if $\|\Phi(x) - \Phi(x_\tau)\|$ is
small when the deformation is small. The deformation size is measured by
$\|\nabla \tau \|_\infty =  \sup_u |\nabla \tau(u)|$. 
If this quantity vanishes then $\tau$ is a ``pure'' translation without
deformation. Stability is formally defined as Lipschitz continuity
relatively to this metric. It means that there exists $C > 0$ such that for all
$x(u)$ and $\tau (u)$ with $\|\nabla \tau \|_\infty < 1$
\begin{equation}
\label{eq:stability} 
\|\Phi(x) - \Phi(x_\tau)\| \leq C\,  \|\nabla \tau \|_\infty\, \|x\|~.
\end{equation}

\begin{figure*}\sidecaption
%\begin{center}
\resizebox{0.7\hsize}{!}{\includegraphics[width=1\linewidth]{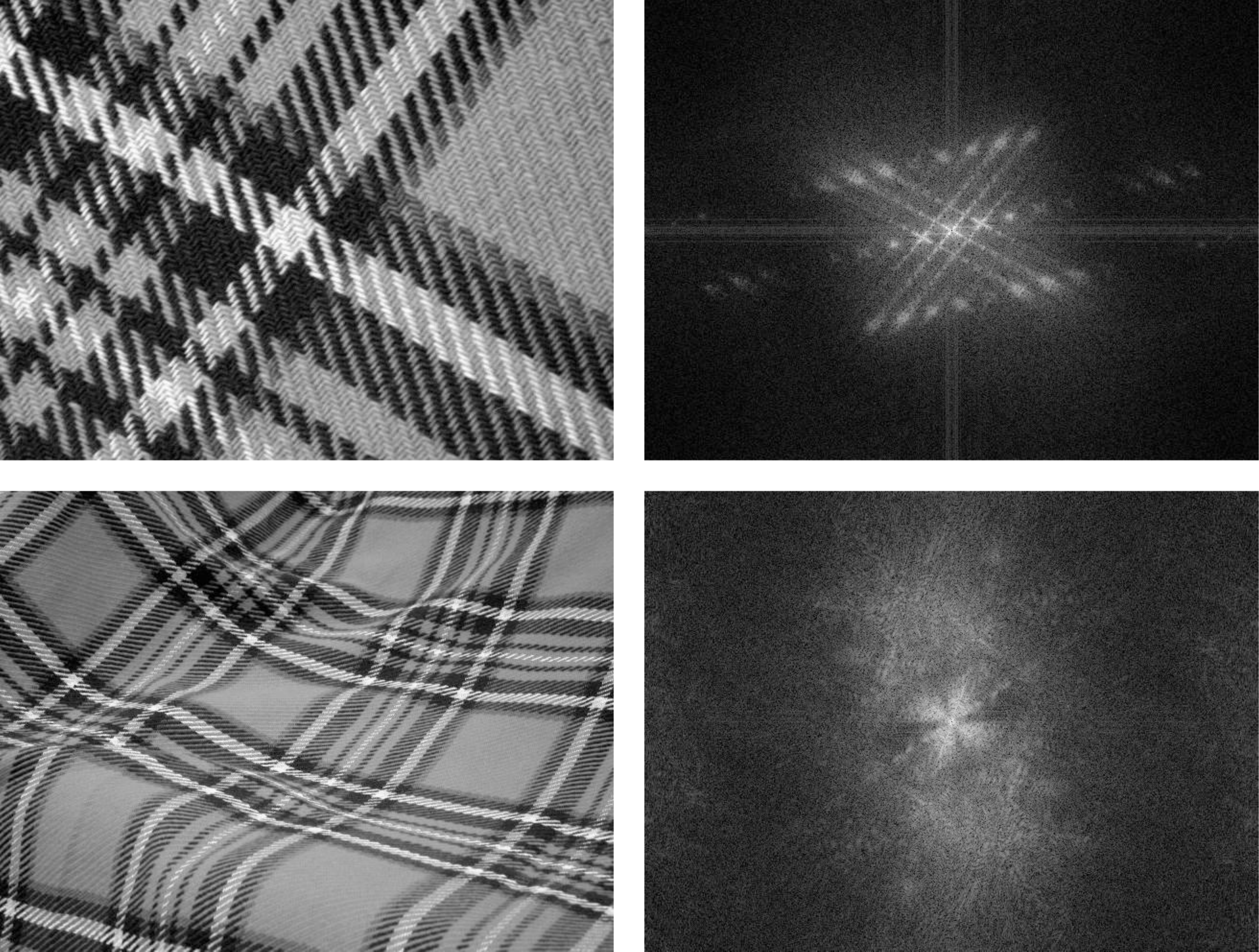}}
\caption{Two images of the same texture (left) from the UIUCTex dataset \citep{UIUC_DL} and the $\log$ of their modulus of Fourier transform (right). The periodic patterns of the texture corresponds to fine grained dots on the Fourier plane. When the texture is deformed, the dots spread on the Fourier plane, which illustrates the fact that modulus of Fourier transform is unstable to elastic deformation.}
\label{fig:fourierunstable}
%\end{center}
\end{figure*}

This Lipschitz continuity property implies that
deformations are locally linearized by the representation $\Phi$.
Indeed, Lipschitz continuous operators are almost everywhere differentiable
in the sense of Gateau. 
It results that $\Phi(x) - \Phi(x_\tau)$ can be approximated by a linear
operator of $\nabla \tau$ if $\|\nabla \tau\|_\infty$ is small. A family of small deformations 
thus generates a linear space $\spn_\tau (\Phi (x_\tau))$. In the transformed space, an invariant
to these deformations can then be computed 
with a linear projector on the orthogonal complement $\spn_\tau (\Phi (x_\tau)) ^\perp$.

Registration invariants are not stable to deformations. 
If $x(u) = 1_{[0,1]^2}(u) + 1_{[\alpha,\alpha+1]^2}(u)$ then for $\tau(u) = \epsilon u$
one can verify that $\|x - x_\tau \| \geq 1$ if 
$|\alpha| > \epsilon^{-1}$.  It results that (\ref{eq:stability}) is
not valid. One can similarly prove that the Fourier transform modulus
$\Phi(x) = |\hat x|$ is not stable to deformations because high frequencies
move too much with deformations as can be seen on figure \ref{fig:fourierunstable}.

Translation invariance often needs to be computed locally.
Translation invariant descriptors which are stable to deformations can be
obtained by averaging. If translation invariant is only needed within
a limited range smaller than $2^J$ then it is sufficient to average
$x$ with a smooth window $\phi_J(u) = 2^{-2J} \phi (2^{-J} u)$ of width $2^J$:
\begin{equation}
\label{eq:localavg}
x \star \phi_J (u) = \int x(v) \phi_J (u-v) \,dv .
\end{equation}
It is proved in \citep{mallat} that
if $\|\nabla \phi \|_1 < +\infty$ and  
$\||u|\,\nabla \phi(u) \|_1 < +\infty$ and
$\|\nabla \tau\|_\infty\leq 1-\epsilon$ with $\epsilon > 0$
then there exists $C$ such that
\begin{equation}
\label{Wave1}
\|x_\tau \star \phi_J - x \star \phi_J \| \leq C\,\|x\| \Bigl(
{2^{-J} \|\tau \|_\infty +  \|\nabla \tau \|_\infty } \Bigr)~.
\end{equation}

Averaging operators lose all high frequencies, and hence eliminate
most signal information. These high frequencies can be recovered with
a wavelet transform. 

\subsection{Wavelet Transform Invariants}
\label{subsec:wavelet2d}
Contrarily to sinusoidal waves, wavelets are localized
functions which are stable to deformations. They are thus well adapted
to construct translation invariants which are stable to deformations.
We briefly review wavelet transforms and their applications in computer vision.
Wavelet transform has been used to analyze stationary processes
and image textures. They provide a set of coefficients closely related
to the power spectrum.

\begin{figure*}\sidecaption
%\begin{center}
\includegraphics[width=0.7\linewidth]{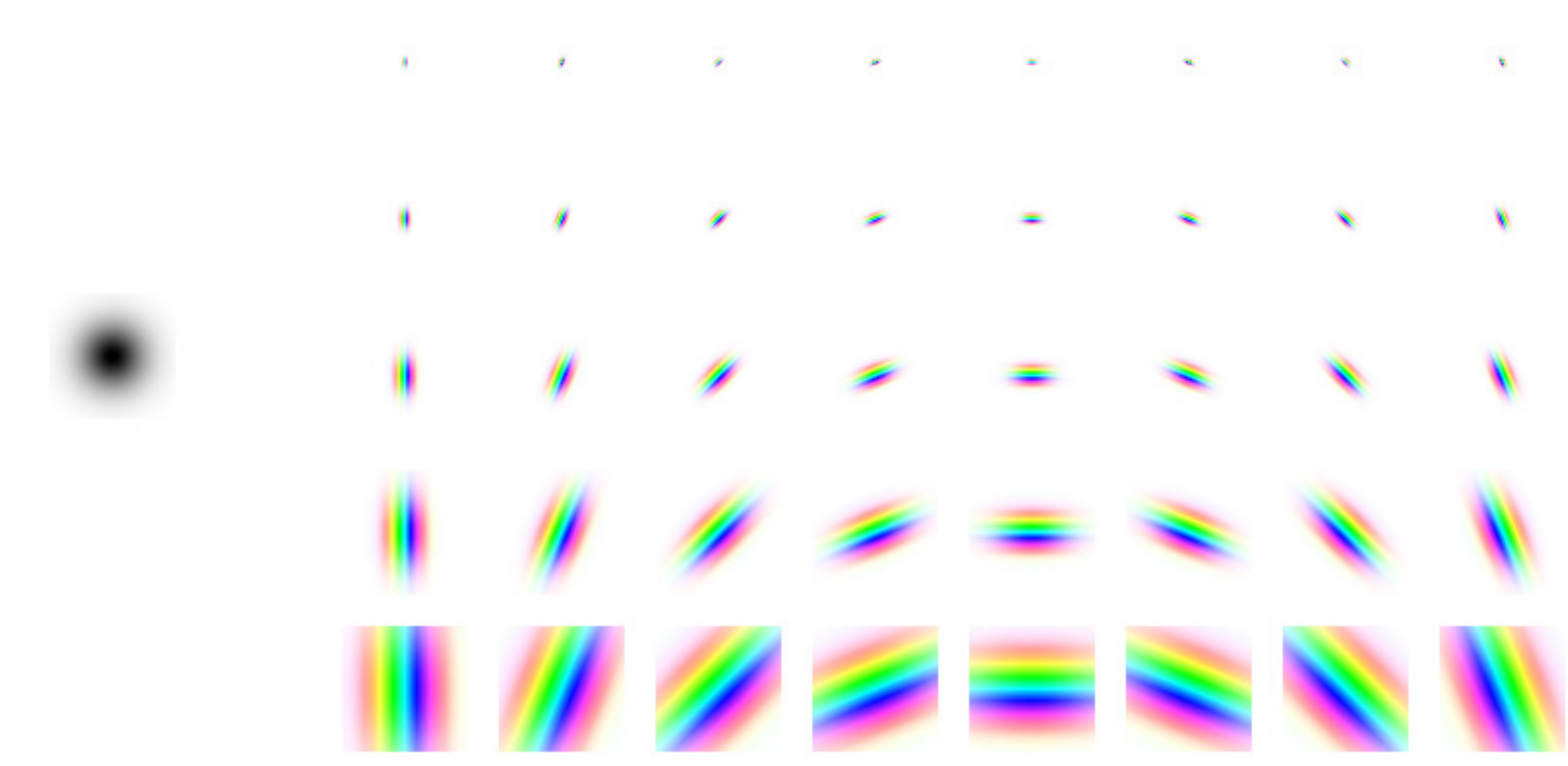}
\caption{The gaussian window $\phi_J$ (left) and oriented and dilated Morlet wavelets $\psi_{\theta,j}$ (right). Saturation corresponds to amplitude while color corresponds to complex phase.}
\label{fig:allwavelet}
%\end{center}
\end{figure*}

A directional wavelet transform extracts the signal high-frequencies
within different frequency bands and orientations.
Two-dimensional directional wavelets are obtained by scaling and rotating
a single band-pass filter $\psi$.
Multiscale directional wavelet filters are defined for any $j \in \Z$
and rotation $r_\theta$ of angle $\theta \in [0,2 \pi]$ by
\begin{equation}
\label{rotdilwave}
\psi_{\theta,j} (u) =  2^{-2j} \psi (2^{-j} r_{-\theta} u)~.
\end{equation}
If the Fourier transform $\hat \psi(\omega)$ is centered at a frequency
$\eta$ then $\hat \psi_{\theta,j} (\omega) = \hat \psi(2^{j} r_{-\theta} \omega)$
has a support centered at
$2^{-j} r_{\theta} \eta$, with a bandwidth proportional to $2^{-j}$. 
We consider a group $G$ of rotations $r_\theta$ which is either a finite
subgroup of $SO(2)$ or which is equal to $SO(2)$. A finite rotation
group is indexed by $\Theta = \{2 k \pi/K~:~0 \leq k < K \}$
and if $G = SO(2)$ then $\Theta = [0,2\pi)$. 
The wavelet transform at a scale $2^J$ is defined by
\begin{equation}
\label{wavedfn}
W x = \Big\{ x \star \phi_J (u)\,,\,x \star \psi_{\theta,j} (u) \Big\}_{u \in \R^2 , \theta \in \Theta , j < J}~.
\end{equation}
It decomposes $x$ along different orientations $\theta$
and scales $2^j$ in the neighborhood of each location $u$.

The choice of wavelet $\psi$ depends upon the desired angular resolution.
In the following we shall concentrate on Morlet wavelets. 
A Morlet wavelet is defined by
\begin{equation}
\psi(u_1,u_2) = \exp \left( - \frac{u_1^2 + u_2^2 / \zeta^2}{2} \right) (\exp(i \xi u_1) - K)
\end{equation}
The slant $\zeta$ of the envelope control the angular sensitivity of $\psi$. The second factor is an horizontal sine wave of frequency $\xi$. 
The constant $K>0$ is adjusted so that $\int \psi = 0$.
Morlet wavelets for $\pi \leq \theta < 2\pi$ are not computed since they verified $\psi_{\theta+\pi,j} = \psi_{\theta,j}^*$, where $z^*$ denotes the complex conjugate of $z$. 
The averaging function is chosen to be a Gaussian window
\begin{equation}
\phi(u) = (2\pi \sigma^2)^{-1} \exp(-u^2/ (2\pi \sigma^2))
\end{equation}
Figure \ref{fig:allwavelet} shows such window and Morlet wavelets.

To simplify notations, we shall write 
$\sum_{\theta \in \Theta} h(\theta)$ a summation over $\Theta$ even
when $\Theta = [0,2\pi)$ in which case this discrete sum represents the
integral $\int_0^{2\pi} h(\theta)\,d\theta$.
We consider wavelets which satisfy the following
Littlewood-Paley condition, for $\epsilon > 0$ and almost all $\om \in \R^2$
\begin{equation}
\label{eq:pars}
1 - \epsilon \leq |\hat \phi(\om)|^2 + 
\sum_{j<0} \sum_{\theta \in \Theta}
|\hat \psi (2^{j} r_\theta \omega)|^2\leq 1~.
\end{equation}
Applying the Plancherel formula proves that
if $f$ is real then 
$W x = \{x \star \phi_{2^J}\,,\,x \star \psi_{\theta,j} \}_{\theta, j}$
satisfies
\begin{equation}
\label{wavecont}
(1 - \epsilon)\, \|x\|^2 \leq \|W x \|^2 \leq \|x\|^2 ~,
\end{equation}
with 
\[
\|W x \|^2 = \|x \star \phi_{2^J} \|^2 + 
\sum_{j<J} \sum_{\theta \in \Theta} \|x \star \psi_{\theta,j} \|^2\,.
\]
In the following we suppose that 
$\epsilon < 1$ and hence that the wavelet transform is a nonexpansive
and invertible operator, with a stable inverse. 
If $\epsilon = 0$ then $W$ is unitary. 

The Morlet wavelet $\psi$ shown in Figure \ref{fig:allwavelet}
together with $\phi(u) = \exp(-|u|^2/(2 \sigma^2))/(2 \pi \sigma^2)$
for $\sigma = 0.7$ satisfy 
(\ref{eq:pars}) with $\epsilon=0.25$. 
These functions are used
in all classification applications.

\begin{figure*}\sidecaption
%\begin{center}
\includegraphics[width=0.7\linewidth]{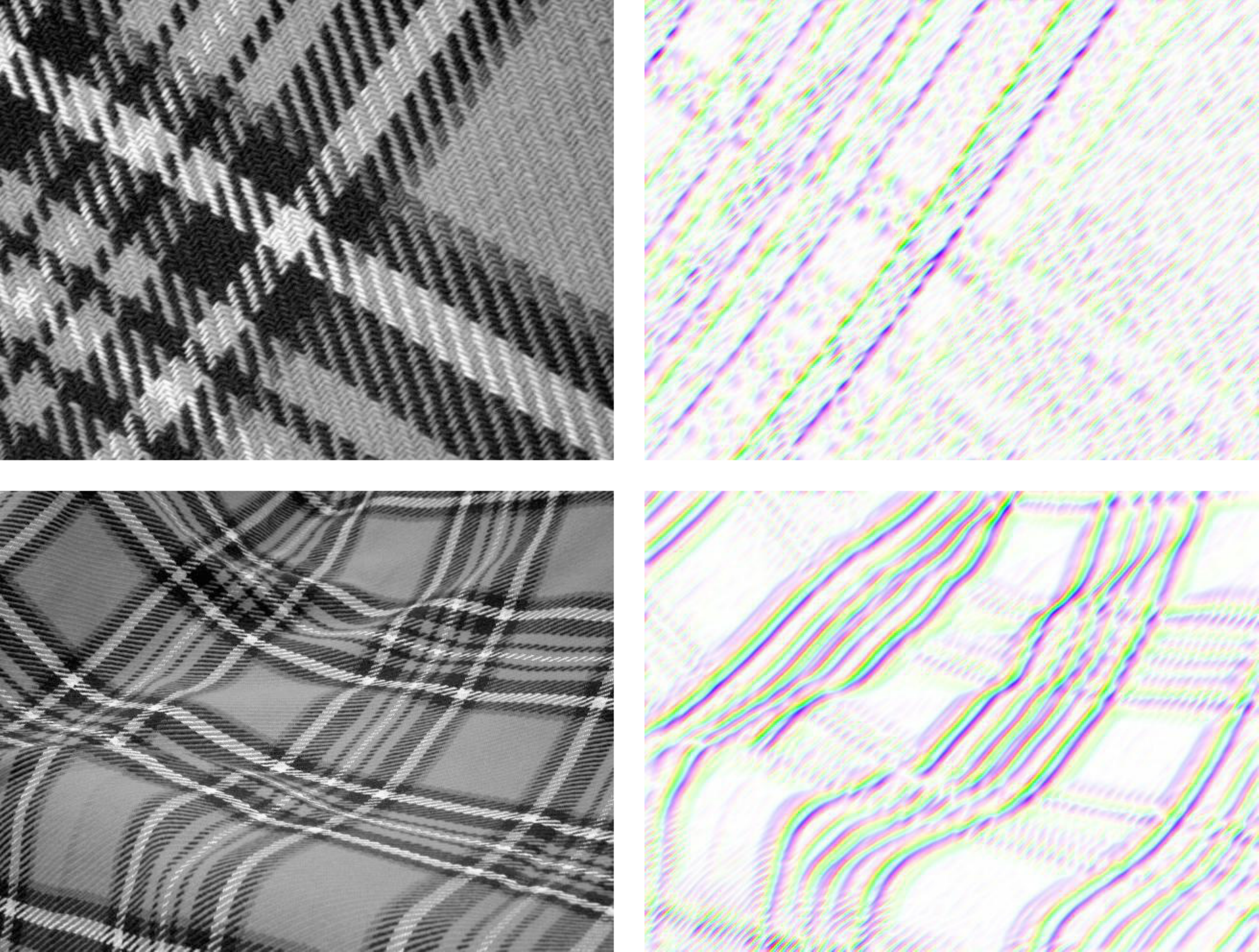}
\caption{Two images of the same texture (left) and their convolution with the same Morlet wavelet (right). Even though the texture is highly deformed, the wavelet responds to roughly the same oriented pattern in both images, which illustrates its stability to deformation.}
\label{fig:waveletstable}
%\end{center}
\end{figure*}

Unlike Fourier waveforms, Morlet wavelets $\psi_{\theta,j}$ are smooth and localized 
which makes them stable to deformation. Figure \ref{fig:waveletstable} shows that 
the responses of the same wavelet for two highly deformed images are comparable
 but displaced. Indeed, Wavelet coefficients $x \star \psi_{\theta,j} (u)$ are 
 computed with 
convolutions.  They are therefore translation covariant, which means that 
if $x$ is translated then $x \star \psi_{\theta,j} (u)$ is translated.
Removing the complex phase like in a Fourier transform defines a positive
envelope $|x \star \psi_{\theta,j} (u)|$ which is still covariant to translation,
not invariant.   Averaging this positive envelope defines locally
translation invariant coefficients which depends upon $(u,\theta,j)$:
\[
S_1 x(u,\theta,j) = |x \star \psi_{\theta,j} | \star \phi_J (u)\,.
\]

Such averaged wavelet coefficients are used under various forms in computer vision. 
Global histograms of quantized filter responses have been used for texture recognition in \citet{maliktex}.
SIFT\citep{sift} and DAISY\citep{daisy} descriptors computes local histogram of orientation. This is similar to $S_1 x$ definition, but with different wavelets and non-linearity. Due to their stability properties, SIFT-like descriptors have been used extensively for a wide range of applications where stability to deformation is important, such as key point matching in pair of images from different view points, and generic object recognition.

\subsection{Transation Invariant Scattering}
\label{subsec:scat2d}

The convolution by $\phi_J$ provides a local translation invariance but also
loses spatial variability of the wavelet transform.
A scattering successively recovers the information lost by the 
averaging which computes the invariants. Scattering consists in a cascade
of wavelet modulus transforms, which can be interpreted as a deep neural network.

A scattering transform is computed by iterating on wavelet transforms
and modulus operators. To simplify notations, we shall write
$\lambda = (\theta,j)$ and
$\Lambda  = \{(\theta,j)~:~\theta \in [0,2\pi]\}$. The wavelet
transform and modulus operations are combined in a single 
wavelet modulus operator defined by:
\begin{equation}
\label{dsfnhdsfs}
|W| x = \Big\{x \star \phi_{J}\,,\,|x \star \psi_{\la}| \Big\}_{\la \in \cPP}~.
\end{equation}
This operator averages coefficients with $\phi_J$ to produce 
invariants to translations and computes higher frequency wavelet transform
envelopes which carry the lost information.
A scattering transform can be interpreted as a neural network illustrated in Figure \ref{fig:translation_scat} which propagates a signal $x$ across multiple layers of the network
and which outputs at each layer $m$ scattering invariant coefficients $S_m x$.

\begin{figure*}\sidecaption
%\begin{center}
\psfrag{lab0}{$x$}
\psfrag{labw0}{$| W|$}
\psfrag{labw1}{$| W|$}
\psfrag{labw2}{$| W|$}
\psfrag{labwm}{$| W|$}
\psfrag{lab2}{$U_1 x$}
\psfrag{lab3}{$U_2 x$}
\psfrag{lab4}{$S_0 x$}
\psfrag{lab5}{$S_1 x$}
\psfrag{lab6}{$S_2 x$}
\psfrag{lab7}{$\dotsc$}
\psfrag{lab8}{$U_m x$}
\psfrag{lab9}{$U_{m+1} x$}
\psfrag{lab10}{$S_m x$}
\includegraphics[width=0.65\linewidth]{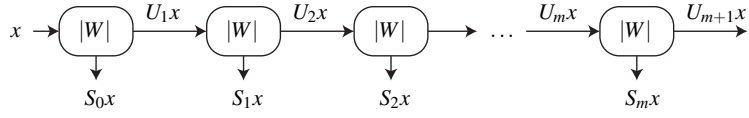}
\caption{Translation scattering can be seen as a neural network which iterates over wavelet modulus operators $|W|$. Each layer $m$ outputs averaged invariant $S_{m} x$ and covariant coefficients $U_{m+1} x$.}
\label{fig:translation_scat}
%\end{center}
\end{figure*}

The input of the network is the original signal $U_0 x = x$.
The scattering transform is then defined by induction. For any $m \geq 0$,
applying the wavelet modulus operator $|W|$
on $U_m x$ outputs the scattering coefficients $S_m x$ and 
computes the next layer of coefficients $U_{m+1} x$:
\begin{equation}
\label{indfsdf}
|W| \,U_{m} x = (S_{m} x \,,\,U_{m+1} x)~,
\end{equation}
with
\begin{eqnarray*}
S_m x (u,\la_1,\dotsc,\la_m) &=& U_m x(.,\la_1,\dotsc,\la_m) \star \phi_J (u)\\
&=&  |\,||x \star \psi_\lau| \star \dotsc |\star \psi_{\la_m} |\star \phi_J(u)
\end{eqnarray*}
and
\begin{eqnarray*}
\label{insdf8sdf8m}
U_{m+1} &x&(u,\la_1,\dotsc,\la_m,\la_{m+1})\\
 &=& |U_m x(.,\la_1,\dotsc,\la_m) \star \psi_{\la_{m+1}}(u)|\\
&=&  |\,||x \star \psi_\lau| \star\dotsc |\star \psi_{\la_m}| \star \psi_{\la_{m+1}}(u)|
\end{eqnarray*}
This scattering transform is illustrated in Figure
\ref{fig:translation_scat}. The final scattering vector concatenates all scattering coefficients
for $0 \leq m \leq \mm$:
\begin{equation}
S x = ( S_m x )_{0 \leq m \leq \mm}.
 \label{eq:scattvec}
\end{equation}
A scattering tranform is a non-expansive operator, 
which is stable to deformations. 
Let $\|S x \| = \sum_m \|S_m x \|^2$, one can prove that
\begin{equation}
\| S x - S y \| \leq \|x - y \|~.
 \label{eq:scadsfnou98}
\end{equation}
Because wavelets are localized and separate scale we can also prove that
if $x$ has a compact support then there exists $C > 0$ such that
\begin{equation}
\label{Wave1}
\|S x_\tau  - S x \| \leq C\,\|x\| \Bigl(
2^{-J} \|\tau \|_\infty +  \|\nabla \tau \|_\infty  
+  \|H \tau \|_\infty  \Bigr)~.
\end{equation}

Most of the energy of scattering coefficients is concentrated on the
first two layers $m = 1, 2$. As a result, applications thus typically
concentrate on these two layers. Among second layer scattering
coefficients
\[
S_2 x (u,\la_1,\la_2) = ||x \star \psi_{\la_1}|\star \psi_{\la_2}|\star \phi_{J}(u)
\]
coefficients $\la_2 = 2^{j_2} r_{\theta_2}$ with $2^{j_2} \leq 2^{j_1}$ have
a small energy. Indeed, $|x \star \psi_{\la_1}|$ has an energy concentrated
in a lower frequency band. As a result, we only compute scattering
coefficients for increasing scales $2^{j_2} > 2^{j_1}$.

\subsection{Separable Versus Joint Rigid Motion Invariants}
\label{subsec:separable_vs_joint}

An invariant to a group which is a product of two sub-groups can be
implemented as a separable product of two invariant operators on each
subgroup. However, this separable invariant is often too strong, and
loses important information. This is shown for translations and rotations.

\begin{figure*}
\begin{center}
%\psfrag{label1}{$| \star \psi_{\theta,j}|$}
%\psfrag{label2}{$| \star \psi_{j_2,\theta_2}|$}
%\psfrag{label3}{$| \star \psi_{\theta,j}|$}
%\psfrag{label4}{$O_j x$}
%\psfrag{label5}{$| \tconv \tpsi_{l_2,j_2,k_2} |$}
%\psfrag{label6}{$x$}
%\psfrag{label7}{$y$}
\includegraphics[width=0.45\textwidth]{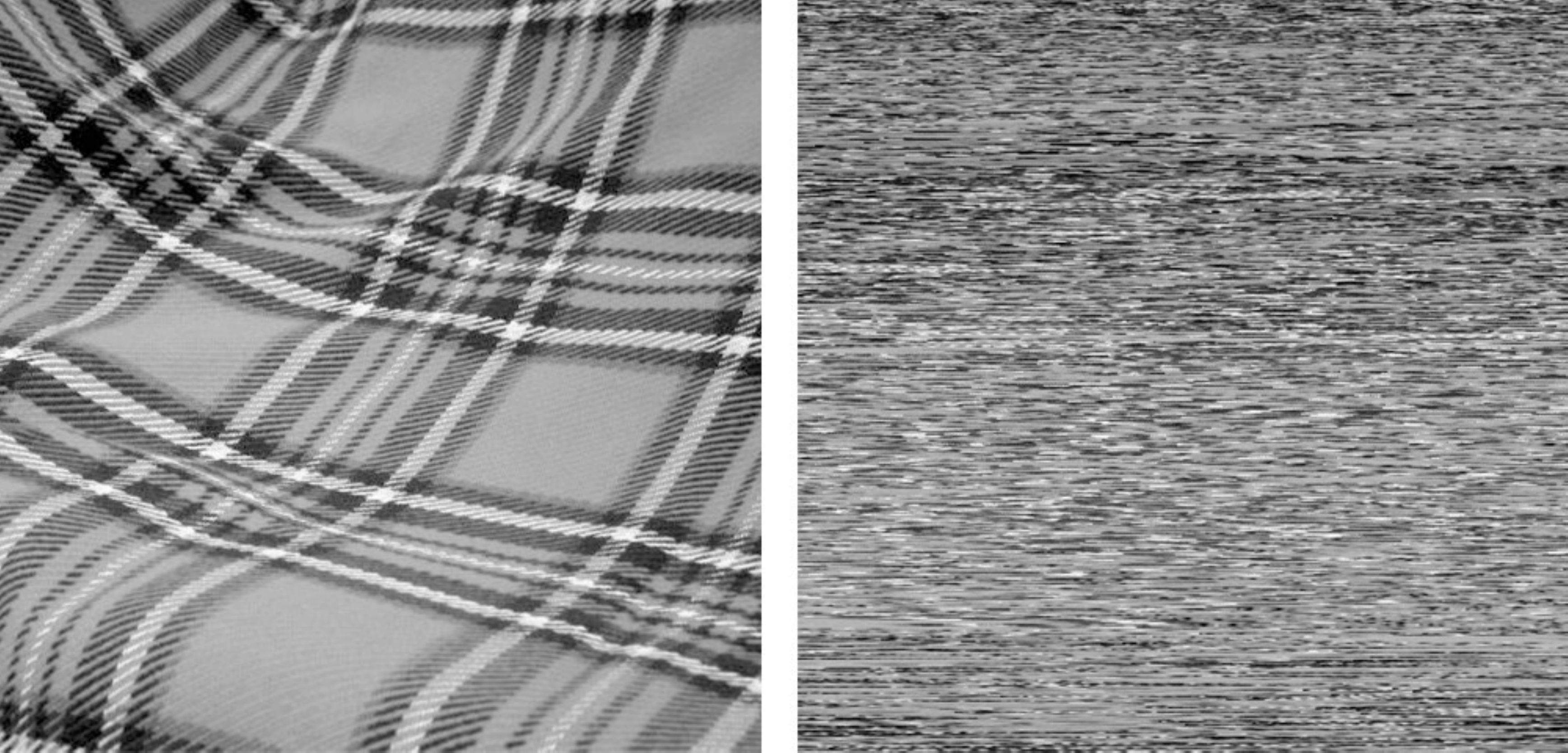}\ \ \ \ \ \ \ \ 
\includegraphics[width=0.45\textwidth]{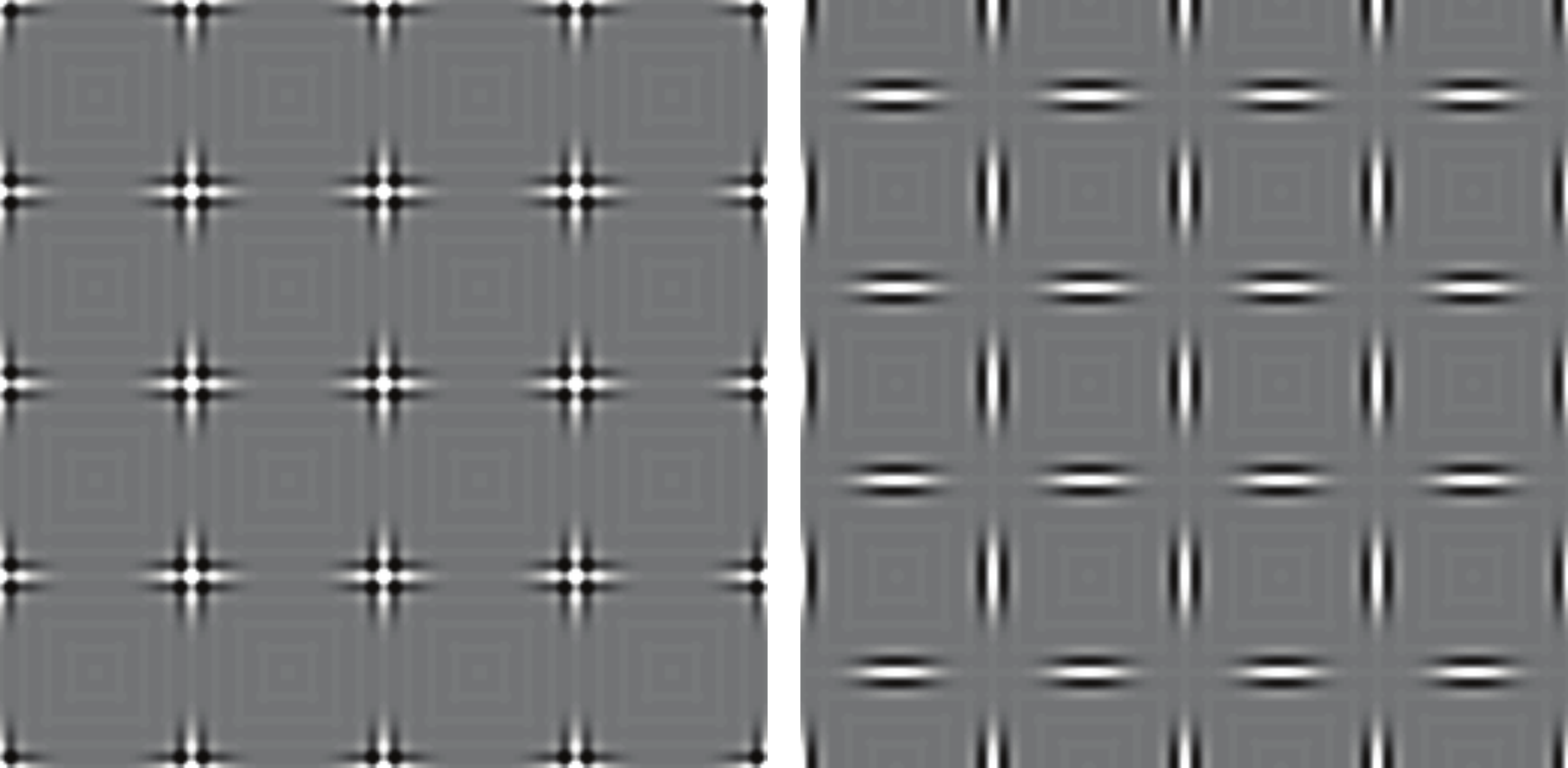}
\end{center}
\caption{(Left) Two images where each row of the second image is translated by a different amount $v(u_1)$. A separable translation invariant that would start by computing a translation invariant for each row would output the same value, which illustrates the fact that such separable invariants are too strong. (Right) Two textures  whose first internal layer is translated by different values for different orientations. In this example, vertical orientations are not translated while horizontal orientations are translated by $1/2(1,1)$. Translation scattering and other separable invariants cannot distinguish these two textures because it does not connect vertical and horizontal nodes.}
\label{fig:joint_vs_sep}
\end{figure*}

To understand the loss of information produced by separable invariants
let us first consider the two-dimensional translation group over $\R^2$. 
A two-dimensional translation invariant operator applied to $x(u_1, u_2)$
can be computed by applying first a translation invariant operator $\Phi_1$ 
which transforms $x(u_1, u_2)$ along $u_1$ for $u_2$ fixed. Then a
second translation invariant operator $\Phi_2$ is applied along $u_2$. 
The product $\Phi_2 \Phi_1$ is thus invariant to any two-dimensional 
translation. However, if $x_v (u_1,u_2) = x(u_1-v (u_2),u_2)$ then 
$\Phi_1 x_v = \Phi_1 x$ for all $v(u_2)$, 
although $x_v$ is not a translation of $x$ because $v(u_2)$ is
not constant. It results that $\Phi x = \Phi x_v$. 
This separable operator is invariant to a much larger
set of operators than two-dimensional translations and can thus
confuse two images which are not translations of one-another,
as in Figure \ref{fig:joint_vs_sep} (left). To avoid this information loss, it
is necessary to build a translation invariant operator which takes into
account the structure of the two-dimensional group. This is why 
translation invariant scattering operators in $\R^2$ are not computed
as products of scattering operators along horizontal and vertical variables.

The same phenomena appears for invariants along translations and rotations,
although it is more subtle because translations and rotations interfere.
Suppose that we apply a translation invariant operator $\Phi_1$, such as
a scattering transform, which separate image components along different
orientations indexed by an orientation parameter $\theta \in [0,2\pi]$.
Applying a second rotation invariant operator $\Phi_2$ which acts along
$\theta$ produces a translation and rotation invariant operator. 

Locally Binary Pattern \citep{lbp} follows this approach. It first builds translation invariance with an histogram of oriented pattern. Then, it builds rotation invariance on top, by either pooling all patterns that are rotated versions of one another, or by computing modulus of Fourier transform on the angular difference that relates rotated patterns.

Such separable invariant operators have the advantage of simplicity
and have thus been used in several computer vision applications. However, as in the separable translation case, 
separable products of translation and rotation 
invariants can confuse very different images.
Consider a first image, which is the sum 
of arrays of oscillatory patterns along two orthogonal directions,
with same locations. If the two arrays of oriented patterns are shifted
as in Figure \ref{fig:joint_vs_sep} (right) we get a very different textures, which
are not globally translated or rotated one relatively to the other.
However, an operator $\Phi_1$ 
which first separates different orientation components and computes a translation invariant 
representation independently for each component will output the same values for both images because
it does not take into account the joint location and orientation structure
of the image. This is the case of separable scattering transforms 
\citep{esann12} or
any of the separable translation and rotation invariant in 
used in \citep{wmfs, lbp}.

Taking into account the joint structure of the rigid-motion group of
rotations and translations in $\R^2$ was proposed by 
several researchers \citep{sarti, remco, boscain}, to preserve image structures
in applications such as noise removal or image enhancement 
with directional diffusion operators  \citep{remco2}.
Similarly, a joint scattering invariant to translations and rotations
is constructed directly on the rigid-motion group in order to take
into account the joint information between positions and orientations.

\section{Rigid-motion Scattering}
\label{sec:roto-trans-scat}

\begin{figure*}
\begin{center}
\psfrag{labu1}{$v_1$}
\psfrag{labu2}{$v_2$}
\psfrag{labth}{$\theta$}
\psfrag{lab1}{$ \tx(v_1, v_2, \theta)$}
\psfrag{lab2}{$ \star y(r_{-\theta}.) $}
\psfrag{lab3}{$ \oconv \oy $}
\psfrag{lab4}{$ \tx \tconv \tiy$}
\includegraphics[width=0.9 \linewidth]{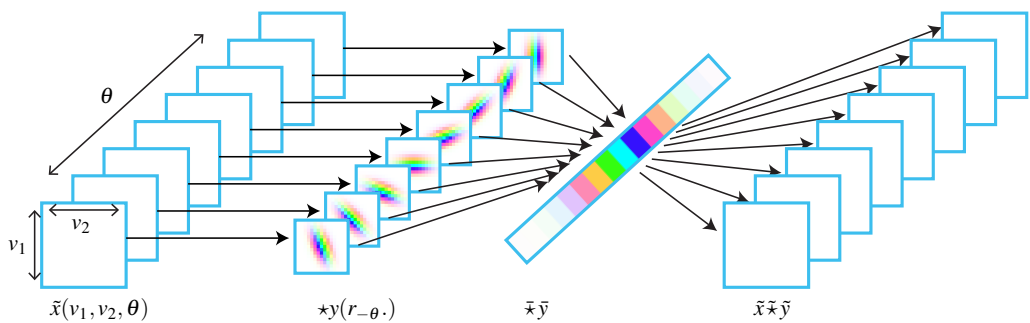}
\caption{A rigid-motion convolution (\ref{dfonsdfsd})
with a separable filter $\tiy (v,\theta) = y(v) \oy(\theta)$
in $\SE$ can be factorized into a two dimensional convolution 
with rotated filters $y(r_{-\theta}v)$ 
and a one dimensional convolution with $\oy(\theta)$.}
\label{fig:3dfilter}
\end{center}
\end{figure*}

Translation invariant scattering operators are extended to 
define invariant representations over any Lie group,
by calculating wavelet transforms on this group.
Such wavelet transforms are well defined with weak conditions on the
Lie group. We concentrate on invariance to the
action of rotations and
translations, which belong to the special Euclidean group.
Next section briefly reviews the properties of the special
Euclidean group. 
A scattering operator \citep{mallat} 
computes an invariant image representation relatively to the action of
a  group by applying wavelet transforms to functions
defined on the group.

\subsection{Rigid-Motion Group}
\label{subsec:affg}

The set of
rigid-motions is called the special Euclidean group $SE(2)$. We briefly
review its properties. 
A rigid-motion  in $\R^2$ is parameterized by a translation $v \in \R^2$
and a rotation $r_\theta \in SO(2)$ 
of angle $\theta \in [0,2 \pi)$.  We write 
$g = (v,\theta)$. 
Such a rigid-motion $g$ maps $u \in \R^2$ to 
\begin{equation}
\label{sdfns}
gu = v + r u \, .
\end{equation}
A rigid-motion $g$ applied to an image $x(u)$ translates and rotates the
image accordingly:
\begin{equation}
g.x (u) =  x(g^{-1} u) = x(r^{-1}(u-v)) \ .
\end{equation}

The group action (\ref{sdfns}) must be compatible with the product $g'.(g u) = (g'.g)u$, so that successive applications of two rigid-motions $g$ and $g'$ are equivalent to the application of a single product rigid-motion $g'.g$. This combined to (\ref{sdfns}) implies that
\begin{equation}
\label{eq:affineproduct}
g'.g = (v' + r_{\theta'} v,\ \theta+\theta') \ .
\end{equation}
This group product is not commutative. 
% pourquoi left inverse et pas juste inverse ?  left inverse = right inverse = inverse when the law is associative which is in the requirement list for "groups".
The neutral element is $(0,0)$, and
the inverse of $g$ is 
\begin{equation}
g^{-1} = (-r_{-\theta} v, {-\theta}).
\end{equation}

The product (\ref{eq:affineproduct}) of 
$\SE$ is the definition of the semidirect product
of the translation group $\R^2$ and the rotation group $SO(2)$:
\[
\SE = \R^2 \rtimes SO(2)~.
\]
It is a Lie group, and the left invariant Haar
measure of $\SE$ is $dg = dv\,d\theta$, obtained as a product of the
Haar measures on $\R^2$ and $SO(2)$. 

The space $\Ld(\SE)$ of finite energy
measurable functions $\tx(v,\theta)$ is a Hilbert space 
\[
\Ld(\SE) = \Big\{ \tx~:~\int_{\R^2}\int_0^{2 \pi} | \tx(v,\theta)|^2\, d\theta dv < \infty \Big\}~.
\]
The left-invariant convolution 
of two functions $\tx (g)$ and $\tiy (g)$ is defined by
\[
\tx \cconv \tiy  (g) = \int_{\SE} \tx (g')\, \tiy  (g'^{-1} g)\, dg'~.
\]
Since $(v',\theta')^{-1} = (-r_{-\theta'} v', {-\theta'})$ 
\begin{equation}
\begin{split}
\tx \cconv & \tiy (v,\theta) =  \\
& \int_{\R^2} \int_{0}^{2 \pi} 
\tx (v',\theta')\, \tiy (r_{-\theta'} (v - v'),\ \theta-\theta') \,dv' d\theta'\ .
\end{split}
\end{equation}
For separable filters $\tiy  (v,\theta) = y(v)\,\oy(\theta)$,
this convolution can be factorized into a spatial convolution with 
rotated filters $y(r_{-\theta} v)$ followed by convolution 
with $\oy (\theta)$:
\begin{equation}
\label{dfonsdfsd}
\begin{split}
\tx &\cconv  \tiy  (v,\theta) = \\
& \int_{0}^{2 \pi} \tx(v',\theta') \int_{\R^2} y(r_{-\theta'} (v - v')) dv' 
\,\oy(\theta-\theta') \, d\theta' ~.
\end{split}
\end{equation}
This is illustrated in Figure \ref{fig:3dfilter}.

\subsection{Wavelet Transform on the Rigid-Motion Group}
\label{subsec:wavelet_transform_roto_trans}

A wavelet transform $\tW$ in $\Ld(\SE)$ is defined as convolutions
with averaging window and wavelets  in $\Ld(\SE)$.
The wavelets are constructed as
separable products of wavelets in $\Ld(\R^2)$ and in $\Ld(SO(2))$.

A spatial wavelet transform in $\LD$ is defined 
from $L$ mother wavelets $\psi_{l} (u)$ which are dilated
$\psi_{l,j} (u) = 2^{-2j} \psi_l(2^{-j} u)$, and a
rotationally symmetric averaging
function $\phi_{J} (u) = 2^{-2J} \phi(2^{-J}u)$ at the maximum
scale $2^J$:
\[
W x = \Big\{ x \star \phi_J (u)\,,\,x \star \psi_{l,j} (u) \Big\}_{u \in \R^2 , 
0 \leq l < L , j < J}~.
\]

Since rotations in $\R^2$ 
parametrized by an angle in $[0,2 \pi]$,  
the space $\Ld(SO(2))$ is equivalent to the space $\Ld[0,2\pi]$.
We denote by $\ox (\theta)$ functions which are $2 \pi$ periodic and belong to
$\Ld(SO(2))$. Circular convolutions of such functions are written
\[
\ox \oconv \oy (\theta) = \int_0^{2 \pi} \ox (\theta')\, \oy(\theta - 
\theta')\, d\theta'~.
\]
Periodic wavelets are obtained by periodizing
a one-dimensional 
scaling function $\phi^1_{K}(\theta) = 2^{-K} \phi^1(2^{-K} \theta)$
and one-dimensional
wavelets $\psi^1_{k} (\theta) = 2^{-k} \psi^1(2^{-k} \theta)$
\begin{eqnarray}
\label{periwavsdfs}
\overline \phi_{K} (\theta) &=& \sum_{m \in \Z} \phi^1_{K} (\theta - 2 \pi m) \\
\overline \psi_{k} (\theta) &=& \sum_{m \in \Z} \psi^1_{k} (\theta - 2 \pi m)~.
\end{eqnarray}
The resulting one-dimensional periodic wavelet tranform
applied to a function $\ox \in \Ld[0,2\pi]$ is calculated with
circular convolutions on $[0,2\pi]$:
\begin{equation}
\label{wavedfn4}
\oW \,\ox = \Big\{ \ox \oconv \ophi_{K} \,,\, \ox \oconv \opsi_{k} \}_{k < K}~.
\end{equation}

A separable wavelet family in $\Ld(\SE)$ is constructed as
a separable product of wavelets in $\Ld(\R^2)$ and wavelets in
$\Ld(\SO)$ 
\begin{equation}
\tphi_{J,K} (v,\theta) =  \phi_{J} (v)  \, \ophi_K (\theta) 
\end{equation}
and for all $ 0 \leq l < L$
\begin{equation}
\tpsi_{l,j,k} (v,\theta) \equaldef 
\left\{
\begin{array}{ll}
 \psi_{l,j} (v)  \ \opsi_k(\theta) & \mbox{if $j < J$ and $k < K$}\\
 \psi_{l,j} (v)  \ \ophi_K (\theta) & \mbox{if $j < J$ and $k = K$}\\
 \phi_{J} (v)  \ \opsi_k (\theta) & \mbox{if $j = J$ and $k < K$}
\end{array}
\right. .
\end{equation}
The resulting wavelet transform is defined by
\[
\tW \tx = \Big\{\tx \cconv \tphi_{J,K} (v,\theta)\,,\,
\tx \cconv \tpsi_{l,j,k} (v,\theta) \Big\}_{l,j,k}~.
\]
Its energy is defined as the sum of the squared $\Ld(\SE)$ norm
of each of its component
\begin{equation}
\label{eq:norm_w2}
\|\tW \tx \|^2 \equaldef \|\tx \cconv  \tphi_{J,K} \|^2 + 
\sum_{l,j,k}
\| \tx \cconv \tpsi_{l,j,k} \|^2 \, .
\end{equation}
The following theorem gives conditions on the one and two-dimensional
wavelets so that $\tW$ is a bounded linear operator which
satisfies an energy conservation.

\begin{theorem}
\label{th:separable_wavelet_affine_group}
If there exists $\epsilon_1 > 0$ and $\epsilon_2 > 0$ such that
\begin{equation}
\label{eq:pars1}
\forall \om \in \R,~~
1 - \epsilon_1 \leq |\hat \phi^1(\om)|^2 +
\sum_{k<0}
|\hat \psi^1 (2^{k} \omega)|^2 \leq 1~,
\end{equation}
\begin{equation}
\label{eq:pars2}
\forall \om \in \R^2,~
1 - \epsilon_2 \leq |\hat \phi(\om)|^2 + 
\sum_{\substack{0 \leq l < L \\ j < 0} }
|\hat \psi_l  (2^{j} \omega)|^2 \leq 1~,
\end{equation}
then
\begin{equation}
\label{opdsfnsdf2}
(1 - \epsilon_1)\,(1-\epsilon_2) \, \| \tx \|^2
 \leq \|\tW \tx \|^2 \leq \|\tx \|^2~.
\end{equation}
\end{theorem}

\begin{proof}
We denote
$Wx(u,J) = x \star \phi_J (u)$, 
$Wx(u,l,j) = x \star \psi_{l,j} (u)$, 
$\oW \ox(\theta,K) = \ox \oconv \ophi_K (\theta)$, 
$\oW \ox(\theta,k) = \ox \oconv \opsi_{k} (\theta)$. 
For $j<J$ and $k < K$, applying the separable convolution
formula (\ref{dfonsdfsd})  to
$\tpsi_{l,j,k} (v,\theta) = \psi_{l,j}(v)\,\opsi_k(\theta)$
proves that
\begin{equation}
\begin{split}
&\tx  \cconv \tpsi_{l,j,k} (v,\theta)  \\
 &=\iint \tx (v',\theta') \psi_{l,j}(r_{-\theta'} (v - v')) dv' 
\,\overline \psi_k(\theta-\theta')  d\theta' \\
&= \iint \tx (r_{\theta'} w,\theta') \psi_{l,j}(r_{-\theta'} v - w) dw
\overline \psi_k(\theta-\theta')  d\theta' .
\end{split}
\end{equation}
A similar result is obtained for all other
$\tpsi_{l,j,k} (v,\theta)$. It proves that
\begin{equation}
\label{factoriznsdf}
\tW \tilde x  = \oW R^{-1} W R \tilde  x
\end{equation}
where $R \tx(v, \theta) = \tx (r_{\theta} v, \theta)$ 
and $W \tx$ computes $W x_{\theta} (v,l,j)$ 
from each $x_{\theta} (v) = \tx (v,\theta)$.
We saw in (\ref{wavecont})  that for all $x$,
$(1 - \epsilon_1) \| x \| \leq \|W x\| \leq \| x \|$. The rotation operator $R$ is
unitary $\|R\| = 1$. We are now going to prove that for all $\ox$,
$(1 - \epsilon_2) \| \ox \|\leq \|\oW \ox \| \leq \| \ox \|$. Since $\tilde W =
\oW R^{-1} W R$ this last inequality will prove (\ref{opdsfnsdf2}).

Let $x_0 (\theta) =\overline x(\theta)$ for $\theta \in [0,2\pi)$ and
$ x_0 (\theta) = 0$ otherwise. Observe that
\[
\oW  \ox(\theta,k) =  \int_{-\infty}^{\infty} x_0(\theta')\, \psi^1_k (\theta-\theta')\,d\theta' = x_0 \star \psi^1_k (\theta)
\]
and 
\[
\oW \ox(\theta,K) =  \int_{-\infty}^{\infty} x_0(\theta') \phi^1_K (\theta-\theta')d\theta' = x_0 \star \phi^1_k (\theta)~,
\]
so that
\[
\|\oW \ox \|^2 = \|x_0 \star \phi^1_K\|^2 + \sum_{k < K}
\|x_0 \star \psi^1_k\|^2 ~
\]
where the norms on the right are norms in $\Ld(\R)$. By applying the
Plancherel formula together with (\ref{eq:pars1}) we verify that
\[
(1 - \epsilon_2) \|x_0\|^2 \leq \|\oW \ox \|^2 \leq \|x_0\|^2
\]
and since $\|x_0\|^2 = \int_0^{2 \pi} |x(\theta)|^2 d \theta = \|\ox \|^2$we
conclude that $(1 - \epsilon_2) \|\ox\| \leq \|\oW \ox\| \leq \| \ox \|$ over
$\Ld[0,2\pi]$. $\Box$
\end{proof}

\begin{figure*}\sidecaption
%\begin{center}
\psfrag{lab0}{$x$}
\psfrag{labw0}{$|W|$}
\psfrag{labw1}{$|\tW|$}
\psfrag{labw2}{$|\tW|$}
\psfrag{labwm}{$|\tW|$}
\psfrag{lab2}{$\tU_1 x$}
\psfrag{lab3}{$\tU_2 x$}
\psfrag{lab4}{$\tS_0 x$}
\psfrag{lab5}{$\tS_1 x$}
\psfrag{lab6}{$\tS_2 x$}
\psfrag{lab7}{$\dotsc$}
\psfrag{lab8}{$\tU_m x$}
\psfrag{lab9}{$\tU_{m+1} x$}
\psfrag{lab10}{$\tS_m x$}
\includegraphics[width=0.63\linewidth]{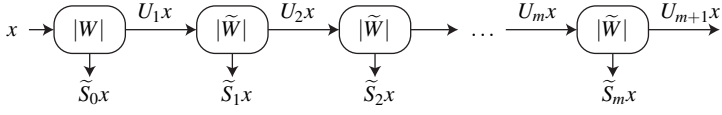}
\caption{Rigid-motion scattering is similar to translation scattering of Figure \ref{fig:translation_scat}, but deep wavelet modulus operators $|W|$ are replaced with rigid-motion wavelet modulus operators $|\tW|$ where convolutions are applied along the rigid-motion group.}
\label{fig:rigidmotion_scat}
%\end{center}
\end{figure*}

\subsection{Rigid-Motion Invariant Scattering Transform}
\label{subsec:scat3d}

A rigid-motion invariant scattering has the same architecture
as the translation invariant scattering of Section \ref{subsec:scat2d}. It is illustrated in Figure \ref{fig:rigidmotion_scat}.
It computes a first spatial wavelet modulus operator $|W|$ and then iterates on rigid-motion wavelet modulus operators $|\tW|$.
To simplify notations, we denote $\la = (\l,j,k)$ 
and $\tLa = \{(\l,j,k)\}$. 
The rigid-motion wavelet modulus operator can be applied
to any function of the rigid-motion group $\tx (g)$ for $g = (u,\theta)$:
\[
|\tW | \tx (g) = 
\Big( \tx \tconv \tphi_{J,K} (g)
\,,\, |\tx \tconv \tpsi_{\la}(g)| \Big)_{\la \in \tLa}
.
\]
Its norm is defined by
\[
\|\tW \tx\|^2 = \|
\tx \cconv \tphi_{J,K} \|^2 + \sum_{\la \in \tLa}
\|\tx \cconv \tpsi_{\la}\|^2~.
\]

The rigid-motion scattering begins with
applying a spatial wavelet modulus operator (\ref{dsfnhdsfs2})
to $x(u)$,
\begin{equation}
\label{dsfnhdsfs2}
|W| x = \Big\{x \star \phi_{2^J}\,,\,|x \star \psi_{\theta,j}| \Big\}_{(\theta,j) \in \cPP}~.
\end{equation}
It computes the first scattering network layer
\[
\tU_1 x(u,\theta,j) = |x \star \psi_{\theta,j}(u)|~.
\]
$\tU_1 x$ is considered as a function of $g = (u,\theta)$, for each
$j$ fixed.
The scattering transform is then defined by induction, with 
successive applications of rigid-motion wavelet modulus transforms
along the $g$ variable. For any $m \geq 1$,
applying the wavelet modulus operator $|\tW|$
on $\tU_m x$ outputs the scattering coefficients $\tS_m x$ and 
computes the next layer of coefficients $\tU_{m+1} x$:
\begin{equation}
\label{indfsdf}
|\tW| \,\tU_{m} x = (\tS_{m} x \,,\,\tU_{m+1} x)~,
\end{equation}
with
\begin{equation*}
\begin{split}
\tS_m  x (g, j_1, & \la_2, \dotsc , \la_m) \\
 &= \tU_m x(.,j_1,\la_2, \dotsc,\la_m) \tconv \tphi_{J,K} (g)\\
&=  |\,||x \star \psi_{.,j_1}| \tconv \tpsi_{\la_2} \dotsc \tconv \tpsi_{\la_m} |\tconv \tphi_{J,K}(g)
\end{split}
\end{equation*}
and
\begin{equation}
\label{insdf8sdf8m}
\begin{split}
\tU_{m+1}  x(& g,j_1, \la_2, \dotsc,\la_m,\la_{m+1}) \\
 & =|\tU_m x(.,j_1,\la_2, \dotsc,\la_m) \tconv \tpsi_{\la_{m+1}}(g)|\\
 & = |\,||x \star \psi_{.,j_1}| \star \tpsi_{\la_2} \dotsc |\star \tpsi_{\la_m}| \star \tpsi_{\la_{m+1}}(g)|
\end{split}
\end{equation}
This rigid-motion scattering transform is illustrated in Figure
\ref{fig:rigidmotion_scat}.

The final scattering vector concatenates all scattering coefficients
for $0 \leq m \leq \mm$:
\begin{equation}
\tS x = ( \tS_m x )_{0 \leq m \leq \mm}.
 \label{eq:scattvec}
\end{equation}

The following theorem proves that 
a scattering transform is a non-expansive operator.

\begin{theorem}
\label{theoinsdfs}
For any $\mm \in \N$ and any $(x,y) \in \LD$
\begin{equation}
\label{non-expansfdonsdf}
\|\tS x - \tS y \| \leq \|x - y\|~.
\end{equation}
\end{theorem}

{\it Proof:} 
A modulus is non-expansive
in the sense that for any $(a,b) \in \C^2$, $||a|-|b|| \leq |a - b|$. 
Since $\tW$ is a linear non-expansive operator, it
results that the wavelet modulus operator $|\tW|$ is also non-expansive
\[
\||\tW| x - |\tW| y \| \leq \| x -  y \| ~.
\]
Since $\tW$ is non-expansive, it results from
(\ref{indfsdf}) that
\begin{equation}
\begin{split}
\| |\tW| \,\tU_{m}  x - &| \tW | \,\tU_{m} y 
\| \\
 & = \|\tS_{m} x - \tS_{m} y 
\|^2 + \|\tU_{m+1} x - \tU_{m+1} y 
\|^2 \\ 
 & \leq 
\|\tU_{m} x - \tU_{m} y \|^2 ~.
\end{split}
\end{equation}
Summing this equation from $m=1$ to $\mm$ gives
\begin{equation}
\label{sdfnsdf2}
\begin{split}
\sum_{m=1}^{\mm} \|\tS_{m} x - \tS_{m} y 
\|^2  + \|\tU_{\mm+1} & x - \tU_{\mm+1} y 
\|^2 \\
 & \leq 
\|\tU_{1} x - \tU_{1} y
\|^2 ~.
\end{split}
\end{equation}
Since $|W| x = (S_0 x , \tU_1 x)$ which is also non-expansive, we get
\begin{equation}
\label{sdfnsdf}
\|S_0 x - S_0 y
\|^2 + \|\tU_1 x - \tU_1 y
\|^2 \leq \|x - y \|^2~.
\end{equation}
Inserting (\ref{sdfnsdf}) in (\ref{sdfnsdf2}) proves (\ref{non-expansfdonsdf}).
$\Box$

\section{Fast Rigid-Motion Scattering}
\label{sec:fawt}

For texture classification applications, first and
second layers of scattering are sufficient for achieving state-of-the-art results. This section describes a fast implementation of 
rigid-motion scattering based on a filter bank implementation of the wavelet
transform.

\subsection{Wavelet Filter Bank Implementation}

\begin{figure*}\sidecaption
%\begin{center}
\psfrag{lab0}{$x$}
\psfrag{labh}{$h \downarrow 2$}
\psfrag{labg1}{$\ g_0$}
\psfrag{labg2}{$\ g_1$}
\psfrag{labg3}{$\ g_2$}
\psfrag{laba1}{$ A_1 x$}
\psfrag{laba2}{$ A_2 x$}
\psfrag{laba3}{$ A_3 x$}
\psfrag{labb10}{$ B_{0,0} x$}
\psfrag{labb20}{$ B_{1,0} x$}
\psfrag{labb30}{$ B_{2,0} x$}
\psfrag{labboo}{$ B_{0,1} x$}
\psfrag{labb21}{$ B_{1,1} x$}
\psfrag{labb31}{$ B_{2,1} x$}
\psfrag{labb12}{$ B_{0,2} x$}
\psfrag{labb22}{$ B_{1,2} x$}
\psfrag{labb32}{$ B_{2,2} x$}
\includegraphics[width=0.60\linewidth]{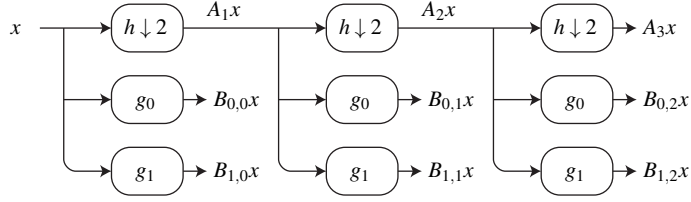}
\caption{Filter bank implementation of the wavelet transform $W$ with $J=3$ scales and $C=2$ orientations. A cascade of low pass filter $h$ and downsampling computes low frequencies $A_j x = x \star \phi_j$ and filters $g_\theta$ compute high frequencies $B_{\theta,j} x = x \star \psi_{\theta, j}$. This cascade results in a tree whose internal nodes are intermediate computations and whose leaves are the output of the downsampled wavelet transform.}
\label{fig:translation_wt}
%\end{center}
\end{figure*}

\label{subsec:fawt}
Rigid-motion scattering coefficients
are computed by applying a spatial
wavelet tranform $W$ and then a rigid-motion wavelet tranform $\tW$.
This section describes filter bank implementations of the spatial wavelet transform.

A wavelet tranform 
\begin{equation}
W x = \Big\{ x \star \phi_J (u)\,,\,x \star \psi_{\theta,j} (u) \Big\}_{u \in \R , \theta \in \Theta , j < J}
\end{equation}
is computed with a filter bank
algorithm, also called ``algorithm \`a trous''. This assumes that the Fourier transform of the window
$\phi(u)$ and each wavelet
$\psi_\theta (u) = \psi(r_{-\theta} u)$ can be written as a product of Fourier transforms of discrete dilated filters $h$ and $g$:
\begin{equation}
\label{nf89sdf}
\hat \phi(\omega) = \prod_{j<0}^{\infty} \hat h(2^{j} \omega)~
\end{equation}
and for all $\theta \in \Theta$
\begin{equation}
\label{nf89sdf_high}
\hat \psi_\theta (\omega) = \hat g_\theta (\omega) \hat \phi(\omega)~.
\end{equation}
Let us initialize $A_0 x = x \star \phi$ and denote $A_j x (n) = x \star \phi_j(2^{j} n)$ and $B_{\theta,j} x (n) = x \star \psi_{\theta,j} (2^j n)$ for $n \in \Z^2$
It results from 
(\ref{nf89sdf}) and (\ref{nf89sdf_high}) that
\begin{eqnarray*}
A_{j+1} x(n) &=& \sum_{p} A_j x(2p)  h(n-2p) \\
B_{\theta,j} x(n) &=& \sum_{p} A_j x(p)  \, g_\theta (n-p)~. 
\end{eqnarray*}
Thus, the subsampled wavelet transform operator can be implemented as a cascade of convolution and downsampling. The convolutions are done with filters $h$ and $g_\theta$ whose support do not change with the scale of the wavelet transform. This allows to use spatial convolutions in the regime where they are faster than FFT-based convolutions.
This is compactly expressed as
\begin{eqnarray*}
A_{j+1} x &=& (A_j x \star h) \downarrow 2 \\ 
B_{\theta,j} x &=& A_j x \star g_\theta \\ 
\end{eqnarray*}
This filter bank cascade is illustrated in Figure \ref{fig:translation_wt}.  Let $N$ be the size of the input image $x$ and $P$ be the size of the filters $h$ and $g_\theta$. A convolution at the finest resolution requires $NP$ operations and $N$ memory. The cascade computes $1+C$ convolutions at each resolution $2^{-j}$. The resulting time complexity is thus $(1+C) \sum_j 2^{-2j} N P = O(CNP)$ and the required memory is $O(CN)$ where $C$ is the number of orientations, $N$ is the size of the input image, and $P$ is the size of the filters $h$ and $g$.

\begin{figure*}
\begin{center}
\psfrag{lab0}{$\tx$}

% filters 
\psfrag{labh}{$h \downarrow 2$}
\psfrag{labg00}{$g_{0,0}$}
\psfrag{labg01}{$g_{0,1}$}
\psfrag{labg10}{$g_{1,0}$}
\psfrag{labg11}{$g_{1,1}$}

\psfrag{laboh}{$ \oh \downarrow 2$}
\psfrag{labog}{$ \og$}

% layers
\psfrag{laba1}{$ A_1 \tx$}
\psfrag{laba2}{$ A_2 \tx$}

\psfrag{labb00}{$ B_{0,0} \tx$}
\psfrag{labb01}{$ B_{0,1} \tx$}
\psfrag{labb10}{$ B_{1,0} \tx$}
\psfrag{labb11}{$ B_{1,1} \tx$}

\psfrag{labc12}{$ C_{1,2} \tx$}
\psfrag{labc22}{$ C_{2,2} \tx$}

\psfrag{labd11}{$ D_{1,1} \tx$}
\psfrag{labd21}{$ D_{2,1} \tx$}

\psfrag{labe002}{$ E_{0,0,2} \tx$}
\psfrag{labe012}{$ E_{0,1,2} \tx$}
\psfrag{labe102}{$ E_{1,0,2} \tx$}
\psfrag{labe112}{$ E_{1,2,1} \tx$}

\psfrag{labf002}{$ F_{0,0,2} \tx$}
\psfrag{labf012}{$ F_{0,2,1} \tx$}
\psfrag{labf102}{$ F_{1,0,2} \tx$}
\psfrag{labf112}{$ F_{1,2,1} \tx$}

\includegraphics[width=0.95\linewidth]{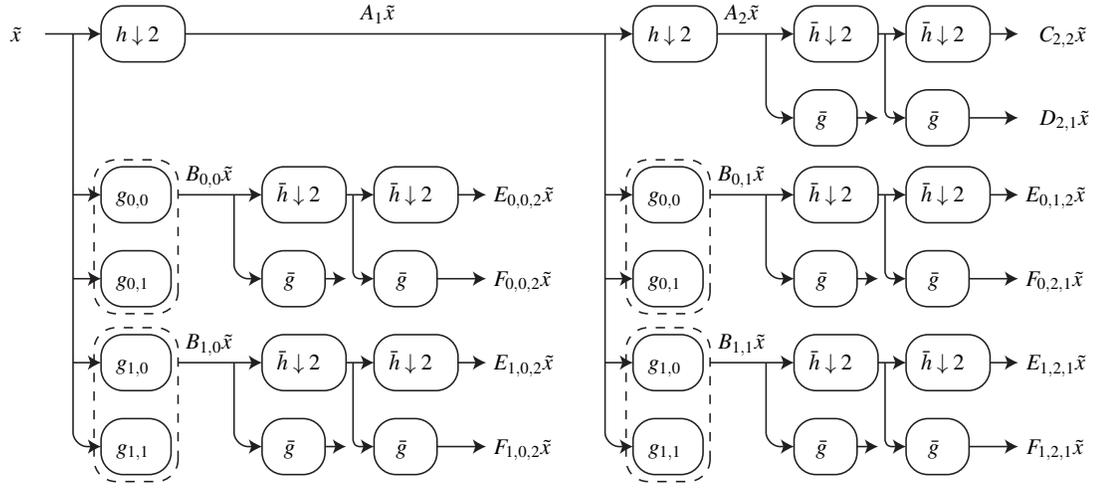}
\caption{Filter bank implementation of the rigid-motion wavelet transform $\tW$ with
 $J=2$ spatial scales, $C=2$ orientations, $L=2$ spatial wavelets, 
 $K=2$ orientation scales . A first cascade computes spatial downsampling	
  and filtering with $h$ and $g_{l,\theta}$. The first cascade 
  is a tree whose leaves are $A_J \tx$ and $B_{l,j}\tx$. 
  Each leaf is retransformed with a second cascade of downsampling and filtering 
  with $\oh$ and $\og$ along the orientation variable. The leaves of the 
  second cascade are $C_{J,K} \tx,\ D_{J,k} \tx$ (whose ancestor is $A_J \tx$) 
  and $E_{l,j,K} \tx,\ F_{l,j,k} \tx$  (whose ancestors are the $B_{l,j} \tx$). 
  These leaves constitute the output of the downsampled rigid-motion wavelet
   transform. They correspond to
    signals $\tx \tconv \tphi_{J,K}\tx,\ \tx \tconv \tpsi_{J,k},\ 
     \tx \tconv \tpsi_{l,j,K},\ \tx \tconv \tpsi_{l,j,k}$ appropriately 
     downsampled along the spatial and the orientation variable.}

\label{fig:rigid_motion_wt}
\end{center}
\end{figure*}
\subsection{Rigid Motion Wavelet Filter Bank Implementation}
Rigid motion wavelet transform $\tW$ takes as input a discretized signal $\tx(n,\theta)$ indexed by position $n$ and orientation $\theta$ and computes a set of convolutions with wavelet $\tW \tx = \{ \tx \tconv \tphi_{J,K},\ \tx \tconv \tpsi_\lambda \}_\lambda$. Similarly to Section \ref{subsec:fawt}, it is computed with two successive cascades of convolution and downsampling along the spatial and orientation variable.  Figure \ref{fig:rigid_motion_wt} illustrates this algorithm.
 We start with the spatial cascade. As previously we initialize $A_{0} \tx = \tx$ and compute
\begin{eqnarray*}
A_{j + 1} \tx  &=& (A_{j} \tx \star h) \downarrow 2  \\ 
B_{l,j} \tx  &=& A_{j} \tx \star g_{l,\theta} 
\end{eqnarray*}
The computation of $B_{l,j} \tx(n,\theta) = (A_j \tx)(., \theta) \star g_{l,\theta}(n)$ involves rotated filters $g_{l,\theta} (n) = g_l(r_{-\theta}n)$ that naturally appear in the factorization (\ref{dfonsdfsd}). There are $LC$ such filters. In our classification experiments, we have chosen to use oriented filters for $g_l$, so that $ g_{l,\theta} = g_{l+\theta}$ and there are only $L = C$ such filters.
The spatial convolution is followed by convolutions along the orientation. Let us denote the subsampled rigid-motion wavelet transform coefficients:
\begin{eqnarray*}
C_{J,K}\tx (n,\theta) &=& \tx \, \tconv \, \tphi_{J,K} (2^Jn, 2^K\theta)  \\
D_{J,k}\tx (n,\theta)  &=& \tx \, \tconv \, \tpsi_{J,k} (2^Jn, 2^k\theta)  \\
E_{l,j,K}\tx (n,\theta)  &=& \tx  \, \tconv \, \tpsi_{l,j,K} (2^jn, 2^K\theta)  \\
F_{l,j,k}\tx (n,\theta)  &=& \tx \, \tconv \, \tpsi_{l,j,k} (2^jn, 2^k\theta)  \ .
\end{eqnarray*}
These subsampled coefficients are initialized from $A$ and $B$ with $ C_{J,0} \tx = A_{J}\tx$ and $E_{l,j,0} \tx = B_{l,j}\tx $. We compute them by induction
\begin{eqnarray*}
C_{J,k+1}\tx &=& (C_{J,k}  \tx \, \oconv \, \oh) \oDownarrow 2 \\
D_{J,k} \tx &=& C_{J,k} \tx \, \oconv \, \og \\
E_{l,j,k+1} \tx &=& (E_{j,l,k} \tx \, \oconv \, \oh) \oDownarrow 2 \\
F_{l,j,k} \tx &=& E_{j,l,k} \tx \, \oconv \, \og \ 
\end{eqnarray*}
where $\oconv$, $\oDownarrow$, $\oh$, $\og$ are the discrete convolution, downsampling, low pass and high pass filters along the orientation variable $\theta$. 

The first spatial cascade computes $CL$ convolutions at each spatial resolution, which requires $O(CLNP)$ operations and $O(CLN)$ memory. Each leaf is then retransformed by a cascade along the orientation variable $\theta$ of cardinality $C$. Convolutions along the orientations are periodic and since the size of the filter $\oh,\ \og$ is of the same order as $C$, we use FFT-based convolutions. One such convolution requires $O(C \log C)$ operations. One cascade of filtering and downsampling along orientations requires $\sum_k C2^{-k} \log(C2^{-k}) = O(C\log C)$ time and $O(C)$ memory. There are $O(LN)$ such cascades so that the total cost for processing along orientation is $O(CLN\log C)$ operations and $O(CLN)$ memory. Thus, the total cost for the full rigid-motion wavelet transform $\tW$ is $O(CLN (P + \log C))$ operations and $O(CLN)$ memory where $C$ is the number of orientations of the input signal, $L$ is the number of spatial wavelets, $N$ is the size of the input image, $P$ is the size of the spatial filters.

\section{Image Texture Classification}
\label{sec:classification}

Image Texture classification has many applications including satellite, medical and material imaging. It is a relatively well posed problem of computer vision, since the different sources of variability contained in texture images can be accurately modeled. This section presents application of rigid-motion scattering 
on four texture datasets containing different types and ranges of variability:  \citep{KTH_TIPS_DL, UIUC_DL, UMD_DL} texture datasets, and the more challenging FMD \citep{fmd,FMD_DL} materials dataset. Results are compared with state-of-the-art algorithms in table \ref{tab:tips}, \ref{tab:uiuc}, \ref{tab:umd} and \ref{tab:fmd}. All classification experiments are reproducible with the ScatNet \citep{scatnet} toolbox for MATLAB.

\subsection{Dilation, Shear and Deformation Invariance with a PCA Classifier }
\label{subsec:pca}

Rigid-motion scattering builds invariance to the rigid-motion group. Yet, texture images 
also undergo other geometric transformations such as dilations, shears or elastic deformations. 
Dilations and shears, combined with rotations and translations, 
generates the group of affine transforms. 
One can define wavelets \citep{shearlet} and a scattering transform on the affine group 
to build affine invariance. 
However this group is much larger and it would involve heavy and unnecessary 
computations. A limited range of dilations and shears is available for finite resolution images which allows one to linearizes these variations.
Invariance to dilations, shears and  deformations are obtained with
linear projectors implemented at the classifier level, by
taking advantage of the scattering's stability to small
deformation.  
In texture application there is typically a small number of 
training examples per class, in which case PCA generative classifiers
can perform better than linear SVM discriminative classifiers \citep{joan}.

Let $X_c$ be a stationary process representing a texture class $c$.
Its rigid-motion scattering transform $\tS X_c$ typically has a power law
behavior as a function of its scale parameters. It is partially linearized
by a logarithm which thus improves linear classifiers. The
random process $\log \tS X_c$ has an 
energy which is essentially concentrated in a low-dimensional affine space
\[
{\bf A}_c  = \E(\log \tS X_c) + {\bf V}_c 
\]
where ${\bf V}_c$ is the principal component
linear space, generated by the eigenvalues of
the covariance of $\log \tS X_c$ having  non-negligible eigenvalues.

The expected value $\E(\log \tS X_c)$ is estimated by the
empirical average $\mu_c$ of the
$\log \tS X_{c,i}$ for all training examples $X_{c,i}$ of the class $c$. 
To guarantee that the scattering moments are partially invariant to
scaling, we augment the training set by dilating each
$X_{c,i}$ by typically $4$ scaling factors $
\{ 1,\ \sqrt{2},\ 2,\ 2\sqrt{2} \}$.
In the following, we consider $\{ X_{c,i} \}_i$ as the set of training
examples augmented by dilation, which
are incorporated in the empirical average estimation $\mu_c$ 
of $\E(\log \tS X_c)$.

The principal components space ${\bf V}_c$ is estimated from 
the singular value decomposition (SVD) of 
the matrix of centered training example $\log \tS X_{i,c}\ -\ \mu_c$.
The number of non-zero eigenvectors which can be computed is equal to the
total number of training examples. We
define ${\bf V}_c$ as the space generated by all eigenvectors. In texture
discrimination applications, it is not necessary to regularize the
estimation by reducing the dimension of this space because there is a small number of
training examples. 

Given a test image $X$, we abusively denote by $\log \tS X$ the average
of the log scattering transform of $X$ and its dilated versions.
It is therefore a scaled averaged scattering tranform, which provides
a partial scaling invariance. We denote by
$P_{{\bf V}_c} \log \tS X$ the orthogonal projection of $\log \tS X$ in the
scattering space ${\bf V}_c$ of a given class $c$. The PCA classification
computes the class $\hat c(X)$ which minimizes the distance 
$\|(Id - P_{{\bf V}_c}) (\log \tS X - \mu_c)\|$
between $\tS X$ and the affine space $\mu_c + \V_c$:
\begin{equation}
\hat{c}(X) = \arg \min_c \| (Id - P_{{\bf V}_c}) (\log \tS X - \mu_c)\|^2
\end{equation}

The translation and rotation invariance of a rigid-motion scattering 
$\tS X$ 
results from the spatial and angle averaging implemented by the
convolution with $\tphi_{J,K}$. It is nearly translation invariant
over spatial domains of size $2^J$ and rotations of angles at most $2^K$. 
The parameters $J$ and $K$ can be adjusted by cross-validation. 
One can also avoid performing any such averaging and let the linear
supervised classifer optimize directly the averaing. This last approaoch
is possible only if there is enough supervised training examples to 
learn the appropriate averaging kernel. 
This is not the case
in the texture experiments of Section \ref{sec:classif} where few training
examples are available, but where the classification task is known to be 
fully translation and rotation invariant. The values of $J$ and $K$ are thus
maximum.

\subsection{Texture Classification Experiments}
\label{sec:classif}

This sections details classification results on image texture datasets KTH-TIPS \citep{KTH_TIPS_DL}, UIUC \citep{uiuc,UIUC_DL} and UMD \citep{UMD_DL}. Those datasets contains images with different range of variability for each different geometric transformation type. We give results for progressively more invariant versions of the scattering and compare with state-of-the-art approaches for all datasets.

\begin{figure}
\includegraphics[width=\linewidth]{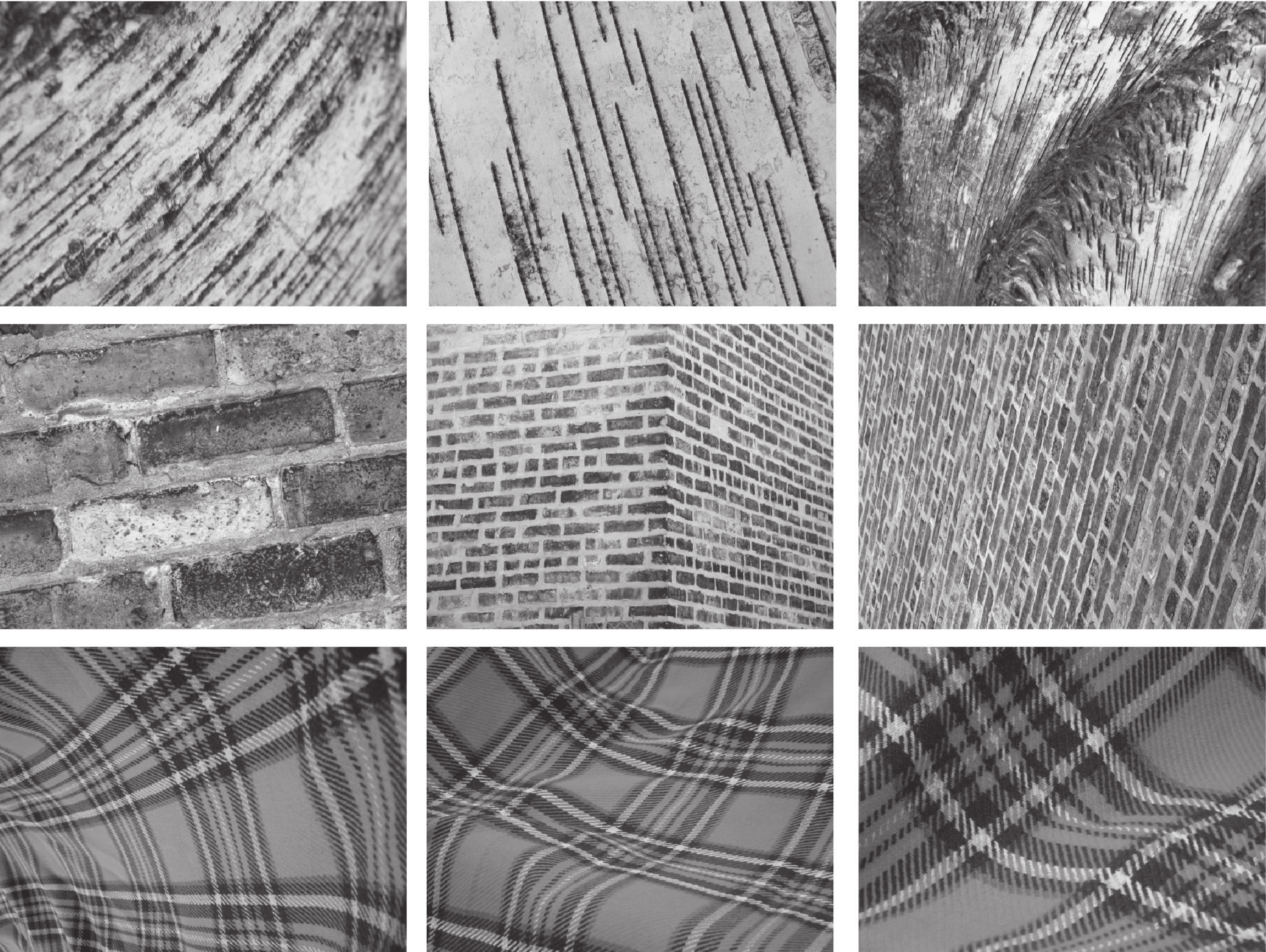}
\caption{Each row shows images from the same texture class in the UIUC database
  \citep{uiuc}, with important rotation, scaling and deformation variability.}
\label{fig:uiuc}
\end{figure}

\begin{table*}
\centering
 \begin{tabularx}{\linewidth}{l | X|X|X}
Train  size                               &  5        &  20         &      40    \\
\hline
 COX \citep{cox}                                          & $80.2 \pm 2.2$     &   $92.4 \pm 1.1$     &    $95.7 \pm 0.5$    \\
  BIF  \citep{bif}                                     &  -           &  -         &  $98.5$ \\  
  SRP \citep{srp} &- &- & $99.3$\\
\hline
Translation scattering   &   $ 69.1 \pm 3.5 $ & $ 94.8 \pm 1.3 $ & $ 98.0 \pm 0.8 $ \\
Rigid-motion scattering    & $ 69.5 \pm 3.6 $ & $ 94.9 \pm 1.4 $ & $ 98.3 \pm 0.9 $ \\
+ $\log$ \& scale invariance & $ {\bf 84.3} \pm 3.1 $ & $ {\bf 98.3}	\pm 0.9 $ & $ {\bf 99.4} \pm 0.4 $ 
 \end{tabularx} 
 \caption{Classification accuracy with standard deviations on \citep{KTH_TIPS_DL} database. Columns correspond to different training sizes per class. The first few rows give the best 
published results. The last rows give results obtained with progressively
refined scattering invariants. Best results are bolded.}
  \label{tab:tips}
\end{table*}

\begin{table*}\sidecaption
 \begin{tabularx}{0.7\linewidth}{l | X|X|X}
Training  size           &  5      &  10    &  20    \\
\hline                                             
Lazebnik \citep{uiuc} & - & $92.6$ &	$96.0$ \\
%COX \citep{cox}  &  91.96 & 95.42  & 97.84 \\
%SRP  \citep{srp}   & - & - & 98.56  \\
 WMFS \citep{wmfs}  & ${\bf 93.4}$ & $97.0$ & $98.6$ \\ 
 BIF  \citep{bif}   & - & - & $98.8 \pm 0.5$  \\
\hline
Translation scattering                & $ 50.0 \pm 2.1 $ & $ 65.2 \pm 1.9 $ & $ 79.8 \pm 1.8$ \\
Rigid-motion scattering       & $ 77.1 \pm 2.7 $ & $ 90.2 \pm 1.4 $ & $ 96.7 \pm 0.8$ \\
+ $\log$ \& scale invariance   & $ 93.3 \pm 1.4 $ & $ {\bf 97.8} \pm 0.6 $ & $ {\bf 99.4} \pm 0.4$ 
 \end{tabularx} 
 \caption{Classification accuracy on \citep{UIUC_DL} database. }
  \label{tab:uiuc}
\end{table*}

\begin{table*}\sidecaption
%\centering
 \begin{tabularx}{0.7\linewidth}{l | X|X|X}
Training size        &  5      &  10    &  20    \\
\hline                                             
 WMFS \citep{wmfs}  & $93.4$ & $97.0$ & $98.7$\\ 
 SRP  \citep{srp}   & - & - & $99.3$ \\
\hline
Translation scattering                        & $ 80.2 \pm 1.9 $ & $ 91.8 \pm 1.4 $ & $ 97.4 \pm 0.9 $ \\
Rigid-motion scattering          & $ 87.5 \pm 2.2 $ & $ 96.5 \pm 1.1 $ & $ 99.2 \pm 	0.5 $ \\
+ $\log$ \& scale invariance       & $ {\bf 96.6} \pm 1.0 $ & $ {\bf 98.9} \pm 0.6 $ & $ {\bf 99.7} \pm 	0.3 $
 \end{tabularx} 
 \caption{Classification accuracy on \citep{UMD_DL} database.}
  \label{tab:umd}
\end{table*}

Most state of the art algorithms use separable invariants to define
a translation and rotation invariant algorithms, and thus lose joint 
information on positions and orientations. This is the case of 
\citep{uiuc} where rotation invariance is obtained through histograms along concentric circles, as well as
Log Gaussian Cox processes (COX) \citep{cox} and Basic Image Features (BIF) \citep{bif} which use rotation invariant patch 
descriptors calculated from small filter responses.
Sorted Random Projection (SRP) \citep{srp} replaces histogram with a similar sorting algorithm and adds
fine scale joint information between orientations and spatial positions
by calculating radial and angular differences before sorting. 
Wavelet Multifractal Spectrum (WMFS) \citep{wmfs} computes wavelet 
descriptors which are averaged in space and rotations, and are
similar to first order scattering coefficients $S_1 x$. 

We compare the best published results \citep{uiuc,cox,bif,wmfs,srp} 
and scattering invariants 
on KTH-TIPS (table \ref{tab:tips}), UIUC (table \ref{tab:uiuc}) and UMD (table \ref{tab:umd}) texture databases.
For the KTH-TIPS, UIUC and UMD database, Tables 1,2,3 give the mean classification accuracy and standard deviation over 200 random splits between training and testing for different training sizes. 
Classification accuracy is computed with scattering representations
implemented with progressively more invariants, and with
the PCA classifier of Section \ref{subsec:pca}. 
As the training sets are small for each class $c$, 
the dimension $D$ of the high variability 
space ${\bf V}_c$ is set to the training size. 
The space ${\bf V}_c$ is thus generated by the $D$ scattering vectors
of the training set.
For larger training databases, it must
be adjusted with a cross validation as in \citep{joan}.

Classification accuracy in Tables 1,2,3 are given for different scattering
representations. The rows ``Translation scattering''  correspond to 
 the scattering described in Section \ref{subsec:scat2d} and initially introduced in \citep{joan}.
The rows ``Rigid-motion scattering'' replace the translation invariant scattering
by the rigid-motion scattering
of Section \ref{subsec:scat3d}. Finally, the rows ``+ $\log$ \& scale invariance'' corresponds to the rigid-motion scattering, with a logarithm non-linearity to linearize scaling, and with the partial scale invariance
described in Section \ref{subsec:pca}, with augmentation at training and averaging at testing along a limited range of dilation.

\begin{table}
\centering
 \begin{tabularx}{\linewidth}{l | X}
Training size        &  50    \\
\hline                                             
 SRP  \citep{srp}   & $48.2$  \\
 Best single feature (SIFT) in \citep{fmd2}& $ 41.2$ \\
\hline
Rigid-motion scattering + $\log$ on grey images          & $51.22 $ \\
Rigid-motion scattering + $\log$ on YUV images & $53.28$
 \end{tabularx} 
 \caption{Classification accuracy on \citep{FMD_DL} database.}
  \label{tab:fmd}
\end{table}

\citep{KTH_TIPS_DL} contains 10 classes of 81 samples with controlled scaling, shear and illumination variations but no rotation. The Rigid-motion scattering does not degrade results but the scale invariant provides significant improvement.

\citep{UIUC_DL} and \citep{UMD_DL} both contains 25 classes of 40 samples with uncontrolled deformations including shear, perspectivity effects and non-rigid deformations. For both these databases, rigid-motion scattering and the scale invariance provide considerable improvements over translation scattering. The overall approach achieves and often exceeds state-of-the-art results on all these databases.

\citep{FMD_DL} contains 10 classes of 100 samples. Each class contains images of the same material manually extracted from Flickr. Unlike the three previous databases, images within a class are not taken from a single physical sample object but comes with variety of material sub-types which can be very different. Therefore, the PCA classifier of Section \ref{subsec:pca} can not linearize deformation and discriminative classifiers tend to give better results. The scattering results reported in table \ref{tab:fmd} are obtained with a one versus all linear SVM. Rigid-motion $\log$ scattering applied to each channel of YUV image and concatenated achieves 52.2 \% accuracy which is to our knowledge the best for a single feature. Better results can be obtained using multiple features and a feature selection framework \citep{fmd2}.

\section{Conclusion}
Rigid motion scattering provides stable translation and rotation invariants through a cascade of wavelet transform along the spatial and orientation variables. We have shown that such joint operators provide tighter invariants than separable operators, which tends to be too strong and thus lose too much information.
A wavelet transform on the rigid-motion group has been introduced, with a fast implementation based on two downsampling and filtering cascade. Rigid-motion scattering has been applied to texture classification in presence of large geometric transformations and provide state-of-the-art classification results on most texture datasets. 

Recent work \citep{edouard} has shown that rigid-motion scattering, with extension to dilation, could also be used for more generic vision task such as object recognition, with promising results on the CalTech 101 and 256 datasets. For large scale deep networks, group convolution might also be useful to learn more structured and meaningful multidimensional filters.

%\begin{acknowledgements}
%If you'd like to thank anyone, place your comments here
%and remove the percent signs.
%\end{acknowledgements}

% BibTeX users please use one of
%\bibliographystyle{spbasic}      % basic style, author-year citations
%\bibliographystyle{spmpsci}      % mathematics and physical sciences
%\bibliographystyle{spphys}       % APS-like style for physics
%\bibliography{}   % name your BibTeX data base

\end{document}